\documentclass{article}

    \usepackage[preprint]{neurips_2020}

\usepackage[utf8]{inputenc} % allow utf-8 input
\usepackage[T1]{fontenc}    % use 8-bit T1 fonts
\usepackage{hyperref}       % hyperlinks
\usepackage{url}            % simple URL typesetting
\usepackage{booktabs}       % professional-quality tables
\usepackage{amsfonts}       % blackboard math symbols
\usepackage{nicefrac}       % compact symbols for 1/2, etc.
\usepackage{microtype}      % microtypography
\usepackage{multirow}
\usepackage{amsthm}
\usepackage{}\usepackage{amsmath,bm}
\usepackage[english]{babel}

\newtheorem{corollary}{Corollary}

\usepackage{indentfirst}
\usepackage{graphicx}
\usepackage{subfig}
\usepackage{caption}
\usepackage{xcolor}
\usepackage{soul}
\usepackage{comment}
\title{TensorCoder: {Dimension-Wise}  Attention via Tensor Representation for Natural Language Modeling}
\author{
  {Shuai Zhang}\textsuperscript{\rm 1}, 
  {Peng Zhang}{\textsuperscript{\rm 1}}{\thanks{Corresponding Author: Peng Zhang}}~~, 
  {Xindian Ma}\textsuperscript{\rm 1},\\
  {\textbf{Junqiu Wei}}\textsuperscript{\rm 2}\textbf{,} 
  {\textbf{Ningning Wang}}\textsuperscript{\rm 1}\textbf{,}  
  {\textbf{Qun Liu}}\textsuperscript{\rm 2}\\
  \textsuperscript{\rm 1}College of Intelligence and Computing, Tianjin University, Tianjin, China\\
  \textsuperscript{\rm 2}Huawei Noah’s Ark Lab, China\\
  \{szhang96, pzhang, xindianma, {w\_ning1215}\}@tju.edu.cn\\
  \{weijunqiu, qun.liu\}@huawei.com  \\}

\begin{document}

\maketitle

\begin{abstract}
Transformer has been widely-used in many Natural Language Processing (NLP) tasks and the scaled dot-product attention between tokens is  a core module of Transformer. This attention is a token-wise design and its complexity is quadratic to the length of sequence, limiting its application potential for long sequence tasks. In this paper, we propose a dimension-wise attention mechanism based on which a novel language modeling approach (namely TensorCoder) can be developed. The dimension-wise attention can reduce the attention complexity from the original $O(N^2d)$ to $O(Nd^2)$, where $N$ is the length of the sequence and $d$ is the dimensionality of head. We verify TensorCoder on two tasks including masked language modeling and neural machine translation. Compared with the original Transformer, TensorCoder not only greatly reduces the calculation of the original model but also obtains improved performance on masked language modeling task (in PTB dataset) and comparable performance on machine translation tasks.

\end{abstract}

\section{Introduction}
In Natural Language Processing(NLP), sequence modeling which embeds a sequence of tokens into continuous vector space is important in language modeling~\cite{sutskever2014sequence} and other downstream tasks~\cite{ramachandran2017unsupervised,zhou2015answer}. In the early work, recurrent neural networks (RNNs, including  LSTMs and GRUs) are often used for sequence learning in many works~\cite{peters2018deep,luong2013better,sundermeyer2012lstm}. Researchers are constantly committed to encoding sequences by 
a longer context. Along with the Transformers~\cite{vaswani2017attention} and their derived pre-trained language models~\cite{devlin2018bert,lan2020albert} emerging, breakthroughs have been made in various NLP tasks. In pursuit of better results,  multi-layers Transformers have been usually adopted but inevitably requires huge computation resource and time consumption~\cite{al-rfou2019character, zhang2019improving}. 
% And in long text modeling task, it also causes a lot of time consumption. 
%Based on the analysis of the Transformer structure, we suggest solving the following two problems would help to improve  existing pre-trained models.

A core module of Transformer is its token-wise attention, which is also one of the main  obstacles to long-sequence modeling. The  complexity of the token-wise attention is $O(N^2d)$, where $N$ is the length of a sequence and $d$ is the dimensionality of head in Transformer. As $N$ increases, both the resource consumption and training time  increase quadratically. In practice, $N$ can be very large. For example,  Kazuki et al.~\cite{irie2019language}  models the longest sentence length of 1343, and Marcin et al.~\cite{junczys-dowmunt2019microsoft} creates sequences of up to 1000 subwords to train Transformer translation models. 

% Second, when building the hidden states of a sequence in a transformer, the attention scores are computed between tokens first, and then the token attention scores are used to encode the hidden states of these tokens. We think that it is not enough to only use token attention scores to build the hidden state of tokens. And the token attention scores are sparse and low in long sequence. The dimension attention score of tokens embedding is also important to encode token. The dimension attention score represents a latent semantic space between tokens based on dimensionality reduction idea. It also removes the noise of attention information. 
Our basic observation is that the dimensionality $d$ can be much smaller than the sequence length $N$. If we can reduce the quadratic complexity to linear complexity by using dimension-wise attention in place of token-wise attention, then we can derive a lighter attention mechanism. To achieve this goal, there are a series of challenging problems. First, how to build such a dimension-wise attention matrix which is linear to $N$ and still has its usefulness in language dependency modeling. Second, how to still 
% encode and retain 
retain the same output form of token-wise attention, which has been proved useful in many earlier studies, 
in such a dimension-wise design. Third, how to 
% extract both the token-wise and dimension-wise features 
extract more effective features in this new attention, and then naturally integrate such a feature extraction process in the Transformer architecture. Finally, how to build a masked language model mechanism, which can prevent the leftward information flow. 

In this paper, we propose a TensorCoder, which is a unified framework that can solve the above problems, based on a novel dimension-wise attention matrix and the tensor representation and operators. First, we can change the multiplication of the query matrix ${\boldsymbol{Q}}$ and value matrix $\boldsymbol{K}$, form $\boldsymbol{QK}^{T}$ (in the token-wise attention) to $\boldsymbol{Q}^{T}\boldsymbol{K}$, yielding a dimension-wise attention matrix, which is linear to $N$ and quadratic to $d$. We also explain this idea using the covariance matrix, which encodes the global dependencies among dimensions, like some latent space models, e.g., principal component analysis (PCA) and latent semantic analysis (LSA), does. In literature, there are  some Transformer variants~\cite{Rewon2019sparse,kitaev2020reformer} which reduce the attention's complexity. However, these variants still focus on the token-wise attention, and have not achieved a linear complexity in the sequence length. 

%Second, we replace token-wise attention in transformer with dimension-wise attention. The token-wise attention use attention information between tokens, but the token attention scores are sparse and low in long sequence. Based on dimensionality reduction idea of Latent semantic analysis(LSI), the dimension-wise attention focuses on the dimension attention score of tokens embedding. The dimension attention score represents a latent semantic space between tokens. It can remove the noise of attention information. There are already some transformer variants~\cite{Rewon2019sparse,kitaev2020reformer} to reduce the complexity of attention. But these transformer variants don't achieve a linearity in sequence length and only focus on the attention information between tokens.
% \textcolor{red}{/*This paragraph is a bit hard to understand. I suggest you directly proposed our idea "replace the token attention in transformers with element-wise attentions", and then explain the reason behind the idea.  It would make it easier to understand.*/}

%The complexity of dimension-wise attention is $O(Nd^2)$, which is linear to the sequence length. It can alleviate the problem of transformer that resource consumption and training time is a quadratic increase in the long-sequence modeling task.

% In order to take into account the token-wise  attention information,
In order to take into account the output form of the token-wise attenion and obtain more dimensional interactive information, we use \textit{Khatri-Rao} (KR) product tensor operation~\cite{li2013some} and construct a third-order tensor representation, which can be proved to have two kinds of final output (implicit representation and explicit representation).
% One of them is same as the token-wise attention that .  
% which can be proved to encode not only explicit attention information between tokens, but also implicit attention information between dimensions of token embeddings. 
Then, we use the convolution method to extract the two kinds of output information from such a third-order tensor, and this convolution method can be naturally integrated into the Transformer mechanism. 
 
 When constructing the dimension-wise attention information of a sequence, the dimension-wise attention uses all the tokens of the entire sequence at the same time. However, in the decoder structure, we should prevent leftward information flow to preserve the auto-regressive property, which is of great necessity for a proper language modeling approach.  We then introduce a masked dimension-wise attention, which ensures that the prediction for the current position $i$ can depend only on the known outputs at positions less than $i$. Such a masked attention, indeed, is also represented and implemented via tensors in our unified TensorCoder framework. 

%Furthermore, when applying TensorCoder to natural language processing tasks(i.e., neural machine translation), the mask mechanism of decoder structure is also an important issue. In order to make dimension-wise attention be used in the decoder structure, we also introduce masked dimension-wise attention based on mask mechanism.
      
% \begin{comment}
% Our major contributions of this paper are as follows:
% \begin{itemize}\setlength{\itemsep}{0pt}
%   \item[1)] We propose a dimension-wise attention mechanism to replace the token-wise attention mechanism in Transformers, and implement a language model named TensorCoder based on the proposed dimension-wise attention mechanism. 
%   \item[2)] Compared to token-wise attention in the standard Transformer, the complexity of dimension-wise attention is $O(Nd^2)$, where $N$ is the sequence length. When $N$ increases, the complexity of dimension-wise attention increases linearly rather than quadratically.
%   \item[3)] In TensorCoder, the third-order tensor constructed by our method is equivalent to the token-wise attention by sum operation, and it also has implicit attention between dimensions of token embedding. In other word, it can simultaneously model the explicit and implicit attention information of the text sequence.
% \end{itemize}
% \end{comment}

To validate the benefits of our model, we test it on two NLP tasks, namely masked language modeling and neural machine translation. In our experiments, the token-wise attention is replaced by the dimension-wise attention. As a result, we show that compared with transformer-based models, TensorCoder considerably reduces the time complexity, and also achieves promising experimental results, especially in 
masked language modeling tasks. In PTB dataset, TensorCoder has achieved a 13\% performance improvement, and its  time cost is only 1/3 of the multi-head token-wise attention. In WMT16 dataset, TensorCoder has achieved comparable results, and the FLOPs of TensorCoder attention are 1/3.5 of the multi-head token-wise attention of Transformer.
% \textcolor{red}{please add more detailed results later}

\section{Preliminaries} 
\label{previous:work}
% The TensorCoder is carried out in this paper. 
% The analysis of dimension-wise attention relies on these concepts about Khatri-Rao Product and Attention mechanism. 
% The analysis of Multi-linear attention is carried out in this paper rely on concepts and results from the filed of tensor decomposition and multi-head attention.
% We cover below in Section~\ref{KR-product} basic background on Khatri-Rao Product~\cite{li2013some}. Then, we describe dot-product attention and  multi-head attention~\cite{vaswani2017attention} in Section~\ref{self-attention}.

% \subsection{Tensor and \textit{Khatri-Rao Product}}
\subsection{Tensor and Khatri-Rao Product}
\label{KR-product}
\textbf{Tensor}  Tensor $\boldsymbol{\mathcal{A}}$ can be thought of as a multi-dimensional array. The order of a tensor is defined to be the number of indexing entries in the array, which are referred to as modes~\cite{zhang2019a}.
% The vector and  matrix also  are called $1$-order tensor and $2$-order tensor, respectively. 
We use a $3$-order tensor that has three modes in the following text. It has column (mode-1), row (mode-2), and tube (mode-3) fibers, denoted by $\boldsymbol{\mathcal{A}}_{:jk}$, $\boldsymbol{\mathcal{A}}_{i:k}$, $\boldsymbol{\mathcal{A}}_{ij:}$ respectively~\cite{li2013some}. In the geometric representation of a tensor, $3$-order tensor can be represented by a cube.

\noindent\textbf{Khatri-Rao Product (KR)} The \textit{Khatri-Rao Product}~\cite{li2013some} is the "matching columnwise" kronecker product. Given matrices $\boldsymbol{M} \in \mathbb{R}^{i \times k}$ and $\boldsymbol{N} \in \mathbb{R}^{j \times k}$, their KR product is denoted by $\boldsymbol{M} \odot \boldsymbol{N}$. 
% It is important to emphasize that the two matrices must have the same dimensions. 
The result is a matrix defined by
\begin{equation}
\label{KR-eq}
 \boldsymbol{\mathcal{A}}=\boldsymbol{M} \odot \boldsymbol{N} = {[\boldsymbol{m}_{1} \otimes
 \boldsymbol{n}_{1}, \boldsymbol{m}_{2} \otimes \boldsymbol{n}_{2}, \ldots, \boldsymbol{m}_{k} \otimes \boldsymbol{n}_{k}]},
\end{equation} 

where $\mathcal{A}\in \mathbb{R}^{(ij) \times k}$, $\boldsymbol{m}_d$ and $\boldsymbol{n}_d$ are the $d$-th column of the matrices $\boldsymbol{M}$ and $\boldsymbol{N}$, respectively, and $\otimes$ is Tensor Product. 
% The \boldsymbol{$m_d \otimes n_d$} is defined by 
% \begin{equation}
% \label{KR-eq}
%  \boldsymbol{m_d \otimes n_d} =
% \begin{pmatrix}
%     m_{1d}n_d \\
%     a_{2d}n_d \\
%     \vdots \\
%     a_{Id}n_d \\
% \end{pmatrix}=
% \begin{pmatrix}
%     m_{1d}n_{1d} \\
%     m_{1d}n_{2d} \\
%     \vdots \\
%     m_{Id}n_{J-1,d} \\
%     m_{Id}n_{Jd}\\
% \end{pmatrix}
% \end{equation}

\subsection{Attention}
\label{self-attention}
\noindent\textbf{Scaled Dot-Product Attention (Token-Wise Attention)} 
% Self-attention is an attention mechanism relating different positions of a single sequence in order to compute a representation of the sequence. 
In Transformer~\cite{vaswani2017attention}, it uses a particular self-attention, called Scaled Dot-Product Attention. The input of attention consists of queries matrix $\boldsymbol{Q}$, keys matrix $\boldsymbol{K}$, and values matrix $\boldsymbol{V}$. The attention can be written as follows:
\begin{equation}
\label{attention}
Attention(\boldsymbol{Q,K,V}) = softmax(\frac{\boldsymbol{QK}^{T}}{\sqrt{d_k}})\boldsymbol{V}
\end{equation}
where matrices $\boldsymbol{Q}$, $\boldsymbol{K}
% \in\mathbb{R}^{N\times{d}}$ 
, \boldsymbol{V} \in \mathbb{R}^{N\times{d}}$, and $d$ is the dimensionality of head. The token-wise attention computes the dot products of the query with all keys, where the time complexity is quadratic in the sequence length. It focuses on the weights between tokens and then applies a softmax function to obtain the weights on the values.

\noindent\textbf{Multi-Head Attention} The Transformer also uses the multi-head attention to allow the model to jointly attend to information from a different representation
subspaces, as introduced in~\cite{vaswani2017attention}:
\begin{equation}
\begin{aligned}
MultiHead(\boldsymbol{Q,K,V}) &=
 Concat(head_1,\ldots,head_k){\boldsymbol{W}^{O}}\\
where~head_i &= Attention({\boldsymbol{QW}^{Q}_{i}, \boldsymbol{KW}^{K}_{i},\boldsymbol{VW}^{V}_{i}})
\end{aligned}
\label{muti-head-attention}
\end{equation}
where matrices ${\boldsymbol{W}^{Q}_{i}}$, ${\boldsymbol{W}^{K}_{i}}$, ${\boldsymbol{W}_{i}^{V}} \in \mathbb{R}^{d_{model}\times{d}}$, ${\boldsymbol{W}^O} \in \mathbb{R}^{hd\times d_{model}}$, $h$ is the number of heads, and $d_{model} = h \times d$ is the dimensionality of model. 

\section{Model Architecture}
In this section, we introduce TensorCoder, which mainly consists of four parts: the dimension-wise attention matrix between dimensions of token embedding, the KR product tensor representation in Figure~\ref{Model}(left), the convolution-based feature extraction in Figure~\ref{Model}(right), and the masked dimension-wise attention in Figure~\ref{mask}.
\begin{figure}[tbp]\small
\centering
\includegraphics[width=4.0in]{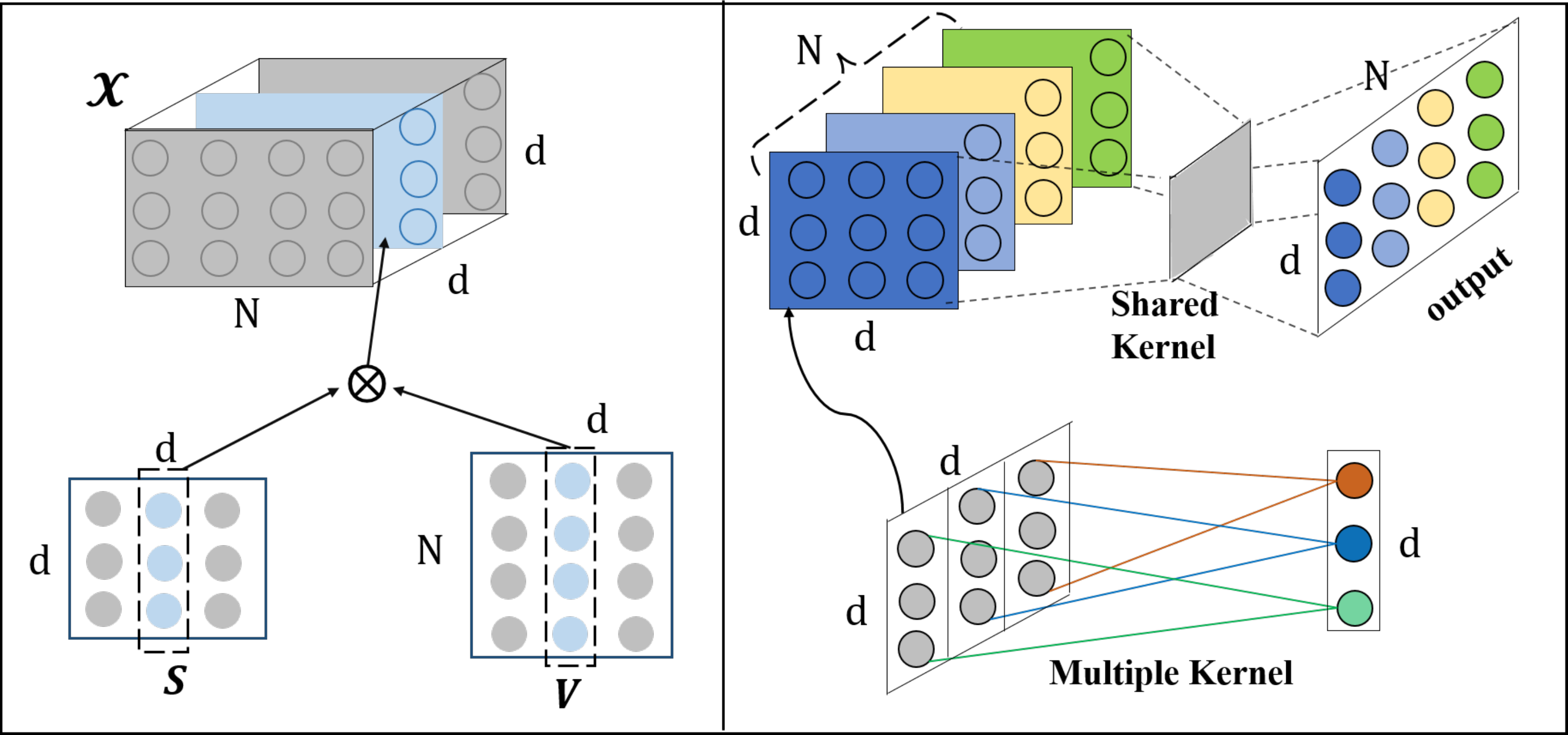}
\caption{(left) Tensor representation using KR product. (right) Feature extraction using convolution.}
\label{Model}
% \vspace{-1px} 
\end{figure}
\subsection{Dimension-Wise Attention Matrix}
To obtain the relationship between dimensions of token embedding, we first compute the attention scores between the dimensions of token embedding with matrices $\boldsymbol{Q}$ and $\boldsymbol{K}$. The scores between all dimensions pack together into a matrix named dimension-wise attention matrix, which is important to construct tensor representation in the next section. The dimension-wise attention matrix $\boldsymbol{S}$ = $\boldsymbol{Q}^{T}\boldsymbol{K}$ is expressed as follows:
\begin{equation}
\label{score}
\begin{aligned}
\boldsymbol{S}_{ij} = \sum_{n=1}^{N} \boldsymbol{Q}_{ni} \boldsymbol{K}_{nj} 
\end{aligned}
\end{equation} 
where matrices $\boldsymbol{Q}$, $\boldsymbol{K}\in\mathbb{R}^{N\times{d}}$, the dimension-wise attention matrix $\boldsymbol{S} \in\mathbb{R}^{d\times{d}}$, $N$ is the length of the sequence, and $d$ is the dimension of the token embedding.
% Each value of attention matrix $\boldsymbol{S^e}$ represents the attention scores or weights between two dimensions by solving the sum of the inner product of the hidden state of all tokens in the two dimensions.
In order to better describe the meaning of the dimension-wise attention matrix , it is necessary to propose Corollary 1. 
\begin{corollary}
\label{corollary1}
In the encoder part, if we standardize and centralize the input for each head, the dimension-wise attention matrix is equivalent to the covariance matrix with the multiplication of two coefficient matrices. It can be expressed as follows:

\begin{equation}
\label{equal1}
\begin{aligned}
\boldsymbol{S} = ({\boldsymbol{W}^q})^T\boldsymbol{CW}^k
\end{aligned}
\end{equation}
where $\boldsymbol{C}$ is the covariance matrix, ${\boldsymbol{W}^q}$ and $\boldsymbol{W}^k$ are coefficient matrices.
\end{corollary}
\begin{proof}
The proof can be found in Supplementary Materials \uppercase\expandafter{\romannumeral1}. 
\end{proof}
In Transformer, the token-wise attention computes the attention matrix by the inner product between each token-pair. It concerns the local correlation between tokens in a sequence. On the other hand, the dimension-wise attention matrix is equivalent to the covariance matrix which describes the linear correlation of feature vectors that represent the semantic space. Compared with the token-wise attention matrix, our method learns the feature representation of each token from a global perspective. In other words, It uses the entire context to encode each token which is crucial in long text modeling tasks. In addition,the size of the dimension-wise attention matrix is relatively lower.

\subsection{KR Product Tensor Representation}
% In order to represent both the token-wise and dimension-wise attention information, 
In order to fuse the interactions between different dimensions and retain the same output representation of the token-wise attention, we use the dimension-wise attention matrix $\boldsymbol{S}$ and matrix $\boldsymbol{V}$ to construct a third-order tensor $\mathcal{\boldsymbol{X}}$. In the process of constructing the third-order tensor, we use the \textit{Khatri-Rao} product~\cite{li2013some}. The specific construction process is as follows:
\begin{equation}
\label{KR}
\begin{aligned}
\boldsymbol{\mathcal{X}} = \boldsymbol{S} \odot \boldsymbol{V}=  {[\boldsymbol{s}_{1} \otimes \boldsymbol{v}_{1}, \boldsymbol{s}_{2} \otimes \boldsymbol{v}_{2}, \ldots, \boldsymbol{s}_{d} \otimes \boldsymbol{v}_{d}]} 
\end{aligned}
\end{equation} 
where $\boldsymbol{\mathcal{X}} \in \mathbb{R}^{N \times d \times d}$, $\boldsymbol{S} \in \mathbb{R}^{d \times d}$, $\boldsymbol{V} \in \mathbb{R}^{N \times d}$, and $\otimes$ is the tensor product. $\boldsymbol{s}_i \in \mathbb{R}^{d}$ and $\boldsymbol{v}_i \in \mathbb{R}^{N}$ are column vectors from matrices $\boldsymbol{S}$ and $\boldsymbol{V}$, respectively. In Figure~\ref{Model}(left), it is a schematic diagram about the \textit{KR} process. The vector ${\boldsymbol{s}_i}$ and ${\boldsymbol{v}_i}$ couple a matrix ( denoted as ${\boldsymbol{s}_i} \otimes {\boldsymbol{v}_i}$ ) by tensor product. Each element of vector ${\boldsymbol{s}_i}$ can be multiplied with each element of vector ${\boldsymbol{v}_i}$, which models more combinations between the two kinds of vectors. The third-order tensor $\boldsymbol{\mathcal{X}}$ can be made by $d$ matrices (${\boldsymbol{s}_i} \otimes {\boldsymbol{v}_i}$), following the steps above. In experiments, we also compute the dimension-wise attention matrix with softmax function to add model nonlinearity.

Based on Eq.~\ref{score} (dimension-wise attention matrix) and Eq.~\ref{KR} (KR product tensor representation), we give a formula of dimension-wise attention as follows:

\begin{equation}
\label{tensor-representation}
\begin{aligned}
\boldsymbol{\mathcal{X}}_{ijk} = f(\boldsymbol{S}_{jk})\boldsymbol{V}_{ik} = \sum_{n=1}^{N}f(\boldsymbol{Q}_{nj}\boldsymbol{K}_{nk})\boldsymbol{V}_{ik}
\end{aligned}
\end{equation} 
where $i \in [N]$ , $j,k \in [d]$, and $f$ means the normalized function (i.e., softmax or scaled factor).

We find that there are two representations (explicit representation and implicit representation) about the output of dimension-wise attention. In explicit representation, the output of dimension-wise attention is represented by the weighted sum of coefficients under a single basis vector of semantic space. The token-wise attention also is the explicit representation. The final output of token-wise attention is represented linearly by the weighted sum of the coefficients under one basis vector. In implicit representation, the output of dimension-wise attention is represented by the weighted sum of coefficients under all basis vectors of semantic space. 

These representation are proved in the following corollary. 
% In the explicit representation, the output of attention only uses one basis vector of semantic space to  . In the implicit representation, the output of attention uses all basis vectors to encode the token from a global perspective. 
The corollary can guide us to choose the appropriate convolution method in feature extraction. We can combine explicit representation and implicit representation to obtain a more efficient final attention output. More details about explicit representation and implicit representation can be seen in Supplementary Materials \uppercase\expandafter{\romannumeral2}. 
% We find the tensor $\boldsymbol{\mathcal{X}}$ not only contains dimension-wise features but also represents the token-wise attention information that refers to the weighted sum of coefficients of different tokens under the same base vector in the paper, as is proved in following the corollary.  
\begin{corollary}
\label{corollary2}
\textcircled{1} By summing over the third-order tensor $\boldsymbol{\mathcal{X}}_{ijk}$ according to the second index $j$, we can obtain a matrix $\boldsymbol{X} \in \mathbb{R}^{N \times d}$, which can be seen as the explicit representation. Both the matrix $\boldsymbol{X}$ and the output of token-wise attention can be represented by the weighted sum of coefficients of different tokens under the same basis vector (e.g., $\boldsymbol{e}_k$):
\begin{equation}
\label{xik}
\begin{aligned}
\boldsymbol{X}_{ik} &= \sum_{j=1}^{d} \boldsymbol{\mathcal{X}}_{ijk} 
% \iff Attention(\boldsymbol{Q, K, V})\\
= \sum_{r=1}^{d} (\delta_r | {\boldsymbol{e}_k)}
\end{aligned}
\end{equation}
\textcircled{2} By summing over the third-order tensor $\boldsymbol{\mathcal{X}}_{ijk}$ according to the third index $k$, we can obtain another matrix $\boldsymbol{X} \in \mathbb{R}^{N \times d}$, which can be seen as the implicit representation. The matrix $X$ can be represented by the weighted sum of coefficients of the token under different basis vectors ($\boldsymbol{e}_1, \ldots, \boldsymbol{e}_d$): 

\begin{equation}
\label{xij}
\begin{aligned}
\boldsymbol{X}_{ij} &=\sum_{k=1}^{d} \boldsymbol{\mathcal{X}}_{ijk} =\sum_{r=1}^{d} (\zeta_r | {\boldsymbol{e}_{r})}
\end{aligned}
\end{equation}
where 
$(\delta_1, \ldots, \delta_d)$ are the different coefficients under the same basis vector $\boldsymbol{e}_k$, and
$(\zeta_1, \ldots, \zeta_d)$ are the coefficients under different basis vector $(\boldsymbol{e}_1, \ldots, \boldsymbol{e}_d)$.

\end{corollary}
\begin{proof}
The proof can be found in Supplementary Materials \uppercase\expandafter{\romannumeral3}. 
\end{proof}

\subsection{Masked Dimension-Wise Attention}
\label{MaskMechanism}
We modify the dimension-wise attention with the mask mechanism in the decoder structure to prevent leftward information flow. In the encoder structure, the dimension-wise attention matrix $\boldsymbol{S}$ is computed by using all the tokens of the entire sequence at the same time (in Eq.~\ref{score}). However, we should avoid using the information of tokens after the current prediction token when calculating the dimension-wise attention score. Different from dimension-wise attention matrix $\boldsymbol{S}$, the masked dimension-wise attention tensor $\boldsymbol{\mathcal{S}}$ uses an upper triangular matrix to mask the tokens behind the current token, which means that the dimension attention score of token $i$ depends only on the known tokens at positions less than $i$. The masked dimension-wise attention tensor $\boldsymbol{\mathcal{S}}$ is expressed as follows:
\begin{equation}
\label{mask score}
\begin{aligned}
\boldsymbol{\mathcal{S}}_{ijk} = \sum_{n=1}^{N} \boldsymbol{Q}_{ni} \boldsymbol{K}_{nj}\boldsymbol{M}_{nk} 
\end{aligned}
\end{equation} 

where $\boldsymbol{\mathcal{S}} \in \mathbb{R}^{d \times d \times N}$ is a third-order tensor. $\boldsymbol{M} \in \mathbb{R}^{N \times N}$ is an upper triangular matrix used to mask off dimension information of the tokens after the current token in calculating dimensions attention score of the current token. The graphic process of the tensor $\boldsymbol{\mathcal{S}}$ is as follows in Figure ~\ref{mask}.

\begin{figure}[tbp]
\centering
\includegraphics[width=3.0in]{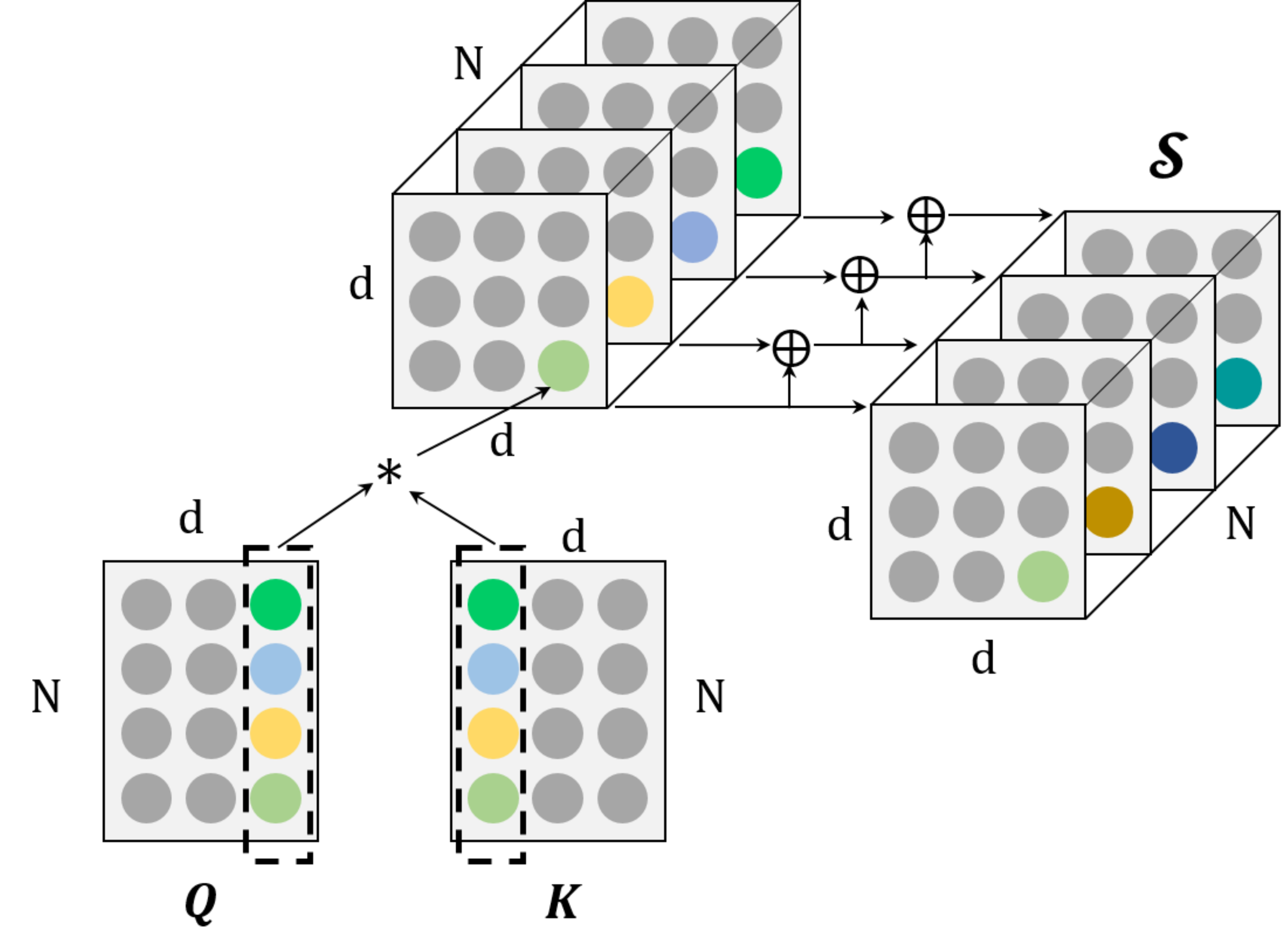}
\caption{The dimension-wise attention tensor $\boldsymbol{\mathcal{S}}$ with mask mechanism. $\ast$ is hadamard product, each column vector of $\boldsymbol{Q}$ and $\boldsymbol{K}$ can generate $N$-dimensional vector by hadamard product, and $\oplus$ is the sum calculation between two matrices.}
\label{mask}
\end{figure}

 We also use the masked dimension-wise attention tensor $\boldsymbol{\mathcal{S}}$ and the $\boldsymbol{V}$ matrix to construct a third-order tensor $\boldsymbol{\mathcal{X}}$. The third-order tensor $\boldsymbol{\mathcal{X}}$ is expressed as follows:
\begin{equation}
\label{mask atention}
\begin{aligned}
\boldsymbol{\mathcal{X}}_{ijk} = \boldsymbol{\mathcal{S}}_{jki} \boldsymbol{V}_{ik}
\end{aligned}
\end{equation} 
where ${\boldsymbol{\mathcal{X}} \in \mathbb{R}^{N \times d \times d}}$, ${\boldsymbol{\mathcal{S}} \in \mathbb{R}^{d \times d \times N}}$, and ${\boldsymbol{V} \in \mathbb{R}^{N \times d }}$. We take broadcast hadamard product between each slice matrix of third-order tensor (sliced by
the last index) and the corresponding row vector of matrix  $\boldsymbol{V}_{i,:}$ to obtain the final third-order tensor $\boldsymbol{\mathcal{X}}$.

\subsection{Feature Extraction}
\label{extraction}
For the third-order tensor $\boldsymbol{\mathcal{X}} \in \mathbb{R}^{N \times d \times d}$, 
we perform sliding convolution filtering along the dimension $N$ of the third-order tensor to obtain final output, which can serve as the input to the next layer network in TensorCoder. The output can be formulated as follows:
\begin{equation}
\label{feature}
\begin{aligned}
\boldsymbol{O}_{ij} = \sum_{m=1}^{d} \boldsymbol{W}_{jm} \boldsymbol{\mathcal{X}}_{ijm} 
\end{aligned}
\end{equation} 
where $\boldsymbol{O} \in \mathbb{R}^{N \times d}$, and  $\boldsymbol{W} \in \mathbb{R}^{d \times d }$ is a parameter matrix, similar to convolution filter. As shown in  Figure~\ref{Model}(right), we first use the filter to get a hidden vector in the slice matrix that is from the third-order tensor splitting in the dimension $N$. The vector usually is called a feature map in computer vision. When filtering in the slice matrix, each column uses different parameters. Then we slide the filter across $\boldsymbol{\mathcal{X}}$ along the dimension ($N$) to obtain final output $\boldsymbol{O}$, which is a collection of different feature maps. In the process of sliding, the filter is shared in different slice matrices.

In Transformer, the multi-head attention uses multiple groups of parameters(i.e., \boldsymbol{$Q, K, V$}) to obtain multiple groups of attention output matrices. Then, these matrices are concatenated and again projected to final output by full connection. In our model, TensorCoder constructs one or more third-order tensors with corresponding groups of parameters and then use multi convolution filters to get different attention outputs. Then, TensorCoder also can get final output from these attention outputs by concat and full connection method, similar to Transformer. 

% ,where the third-order tensor $X$ contains the dimensional attention information using only the dimensional information of the tokens that precedes its position.
\vspace{-0.3cm}
\subsection{Analysis of Complexity}
\label{compressanalysis}
% \textbf{Complexity} 
From Eq.~\ref{attention}, we know the time complexity of the token-wise attention is ${O}(N^2d)$. $N$ is the length of a sequence, and $d$ is the dimension of token embedding. The complexity analysis of our model is introduced from two aspects: first, in the encoder part of TensorCoder, the time complexity of dimension-wise attention matrix (in Eq.~\ref{score}), tensor representation by KR product (in Eq.~\ref{KR}) , and feature extraction (in Eq.~\ref{feature}) all are ${O}(Nd^2)$. Therefore, the time complexity of TensorCoder's encoder part is ${O}(Nd^2)$. In some tasks (i.e., Masked language modeling), they only use the encoder part, TensorCoder has a great advantage compared with Transformer; Second, in the decoder of TensorCoder, the complexity of dimension-wise attention tensor $\mathcal{S}$ (in Eq.~\ref{mask score}) is $O(N^2d^2)$, the complexity of third-order tensor construction process (in Eq.~\ref{mask atention}) and final feature extraction both are ${O}(Nd^2)$. Therefore,  the complexity of TensorCoder is ${O}(Nd^2 + N^2d^2)$ in the decoder part.   
% \vspace{-0.2cm}

% \section{Related Work}
%  Transformer has attracted researchers' attention since it is proposed because it can allow modeling of dependencies regardless of their distance in the long sequence. The Transformer  architecture~\cite{vaswani2017attention} based solely on the self-attention mechanism has been used in many tasks, such as in the natural language processing task~\cite{raffel2019exploring,khandelwal2019sample}, music task~\cite{huang2019music} and image task~\cite{parmar2018image}.
% %  and it also shows excellent performance in these task.
% Furthermore, it plays a basic component in the pre-trained language models that obtain new SoTA on some NLP tasks, such as Bert~\cite{devlin2018bert}, XLNet~\cite{yang2019xlnet}, and Albert~\cite{lan2020albert}.

\section{Comparison with Recent Related Works}
Transformer has some defects in the practical application, such as high attention mechanism complexity, huge
model parameters. Many works have been proposed to solve these problems.
% So many Transformer variants have been proposed one after another. % To solve transformer is limited by the fixed-length context in the language modeling settings, transformer-xl~\cite{dai2019transformer}uses a segment-level recurrence mechanism and a novel positional encoding scheme to enable learning dependency beyond a fixed length without disrupting temporal coherence.
% A large number of parameters in transformer practical application~\cite{shazeer2018mesh} brings huge resources and time consumption, which limits the development of the transformer. 
Ma \textit{et al.}~\cite{ma2019a} propose multi-linear attention to largely compress the model parameters while obtain performance improvements. However, the multi-linear attention  can only be applied to the encoder structure of the Transformer, but cannot replace the decoder's attention mechanism because it lacks the mask mechanism.
% Universal Transformers~\cite{dehghani2019universal} propose a self-attentive recurrent sequence model that uses recurrent models to replace the stacking structure of Transformer. The attention mechanism has not been changed in Universal Transformers. 
Many Research works try to reduce  the complexity of the attention mechanism  to allow Transformer to handle longer sequences. For example, Sparse Transformer~\cite{Rewon2019sparse} introduces a sparse factorizations of the attention matrix method that changes the attention  complexity from $O(N^2)$ to $O(N\sqrt{N})$, where $N$ is the length of the sequence. The Reformer~\cite{kitaev2020reformer} replaces dot-product attention by using the locality-sensitive hashing, and reduces the attention complexity to be $O(N\log N)$. Although these methods have achieved relatively good results, the attention complexity of these methods is still not linear in sequence length. Linear transformer ~\cite{katharopoulos2020transformers} discards the softmax operation in token-wise attention to achieve time and memory complexity $O(N)$. It uses a feature representation for the query matrix, keys matrix, and values matrix, which simplifies the attention mechanism by making use of the associative property of matrix multiplication.

\section{Experiments}
In order to verify the effectiveness of TensorCoder, we carry out two NLP tasks named masked language modeling (MLM) and neural machine translation (NMT). They are considered touchstone and challenging in the NLP field. In the experiment, we take the open-source Transformer model architecture and make changes on it by using dimension-wise attention to replace the multi-head attention of Transformer. 

% Other details (such as Hardware and Optimization) can be found in Supplementary Materials \uppercase\expandafter{\romannumeral4}.

In the experiment, we choose two models as the baseline, Transformer and Tensorized Transformer~\cite{ma2019a}, respectively. Tensorized Transformer is a typical model constructed by the tensor method in Transformer, so that we make it as one baseline. In the experiment, TensorCoder is with a single attention head. 'TensorCoder-1conv' and 'Transformer-1head' both use a single head. 'TensorCoder-8conv' uses eight convolution kernels to obtain eight attention output matrices $\boldsymbol{O}$ (in Eq.~\ref{feature}) from a third-order tensor. 'Transformer-8head' obtains eight attention output matrices by multi-head mechanism (in Eq.~\ref{muti-head-attention}). They both get the same amount of attention outputs, which is a reasonable comparison. Some detailed description can be seen in Section~\ref{extraction}. The Sparse Transformer~\cite{Rewon2019sparse} carried out for natural images, text (character-level), and raw audio. Similar to Sparse Transformer, the Reformer~\cite{kitaev2020reformer} ran experiments on the imagenet64 and enwik8-64K tasks. Our experiments are carried out for word-level language model and neural machine translation, which are not suitable for Sparse Transformer and Reformer.
% A reproducible experiment code is in the Supplementary Material \uppercase\expandafter{\romannumeral5}.
% Complete code\footnote{https://github.com/szhangtju/The-compression-of-Transformer} for running experiments have be released, and the key code which is about our method can be found in Supplementary Materials F.
% \vspace{-0.2cm}
\subsection{Evaluation on Masked Language Modeling}
Masked language modeling (MLM) \cite{devlin2018bert} is the task that masks a certain percentage of the input tokens randomly and then predicts only those masked token. It is often referred to as a Cloze task\cite{taylor1953cloze} in the literature, which is different from the standard conditional language models that can only be trained left-to-right or right-to-left. Similar to BERT~\cite{devlin2018bert}, we mask 15\% of all tokens in each sequence at random in our experiments. Additionally, we do not always replace 'masked' words with the actual [MASK] token. In these 'masked' words, there are three  processing methods: 80\% to replace the 'masked' word with the [MASK] token, 10\% to replace the 'masked' word with a random word, 10\% to keep the word unchanged. The purpose of these methods is to bias the model towards the actual observed word.

We chose two datasets:  PTB  and WikiText-103. PTB has $929k$ training tokens, $73k$ validation words, and $82k$ test words. It is widely used dataset in language model learning. WikiText-103 dataset contains 267,735 unique tokens. The dataset is well suited for models that can take advantage of long term dependencies. It also features a far larger vocabulary and retains the original case, punctuation, and numbers. It contains $103M$ training tokens from $28.5k$ articles, no processing is needed other than replacing newlines with <eos> tokens.

Models are evaluated based on Negative Log-Likelihood Loss (NLL), which  is also called multi-class cross entropy. The smaller the loss is, the better the model performance. To compare the complexity of the model, we measure the efficiency of the model in terms of the number of floating-point operations (FLOPs)~\cite{molchanov2017pruning}. The FLOPs are mainly the sum of multiplication and addition times. The MLM task only uses the encoder structure of Transformer, so that we use dimension-wise attention to replace multi-head attention in each layer of encoder, while other parts remain the same.

% \begin{table*}[ht]\small
% \centering 
% \caption{Results (PPL) and model parameters with state-of-the-art results on One-Billion. Tensorized Transformer is our model. The core-1 is that the model use Single-block term tensor. Analogously, the core-2 is that two block term tensor is used.}
% \begin{tabular}{ccccc}
% \toprule[1pt]
% \textbf{Model} & \textbf{Params} & \textbf{Test PPL}\\
% \midrule[0.5pt]
% RNN-1024+9 Gram~\cite{chelba2013one} & 20B & 51.3\\
% LSTM-2018-512~\cite{jozefowicz2016exploring} & 0.83B & 43.7 \\
% GCNN-14 bottleneck~\cite{dauphin2017language} &--& 31.9 \\
% LSTM-8192-1024+CNN Input~\cite{jozefowicz2016exploring}  &1.04B&30.0 \\
% High-Budget MoE~\cite{shazeer2017outrageously} &5B &28.0 \\
% LSTM+Mos~\cite{yang2017breaking} & 113M & 37.10 \\
% \hline
% Transformer+adaptive input~\cite{baevski2018adaptive} & 0.46B & 23.7 \\
% Transformer-XL Base~\cite{dai2019transformer} &0.46B& 23.5 \\
% Transformer-XL Large~\cite{dai2019transformer} & 0.8B& 21.8 \\
% \hline 
% Tensorized Transformer core-$1$ &0.16B& 20.5 \\
% Tensorized Transformer core-$2$ &0.16B& \textbf{19.5} \\
% \bottomrule[1pt]
% \end{tabular}
% \label{Tabel1}
% % \vspace{-11px}
% \end{table*}

% \subsection{Results and Details}

\begin{figure}[tbp]\small
    \centering
    \subfloat[PTB Loss]{
    \begin{minipage}[t]{0.5\linewidth}
    \centering
    \includegraphics[width=2.5in]{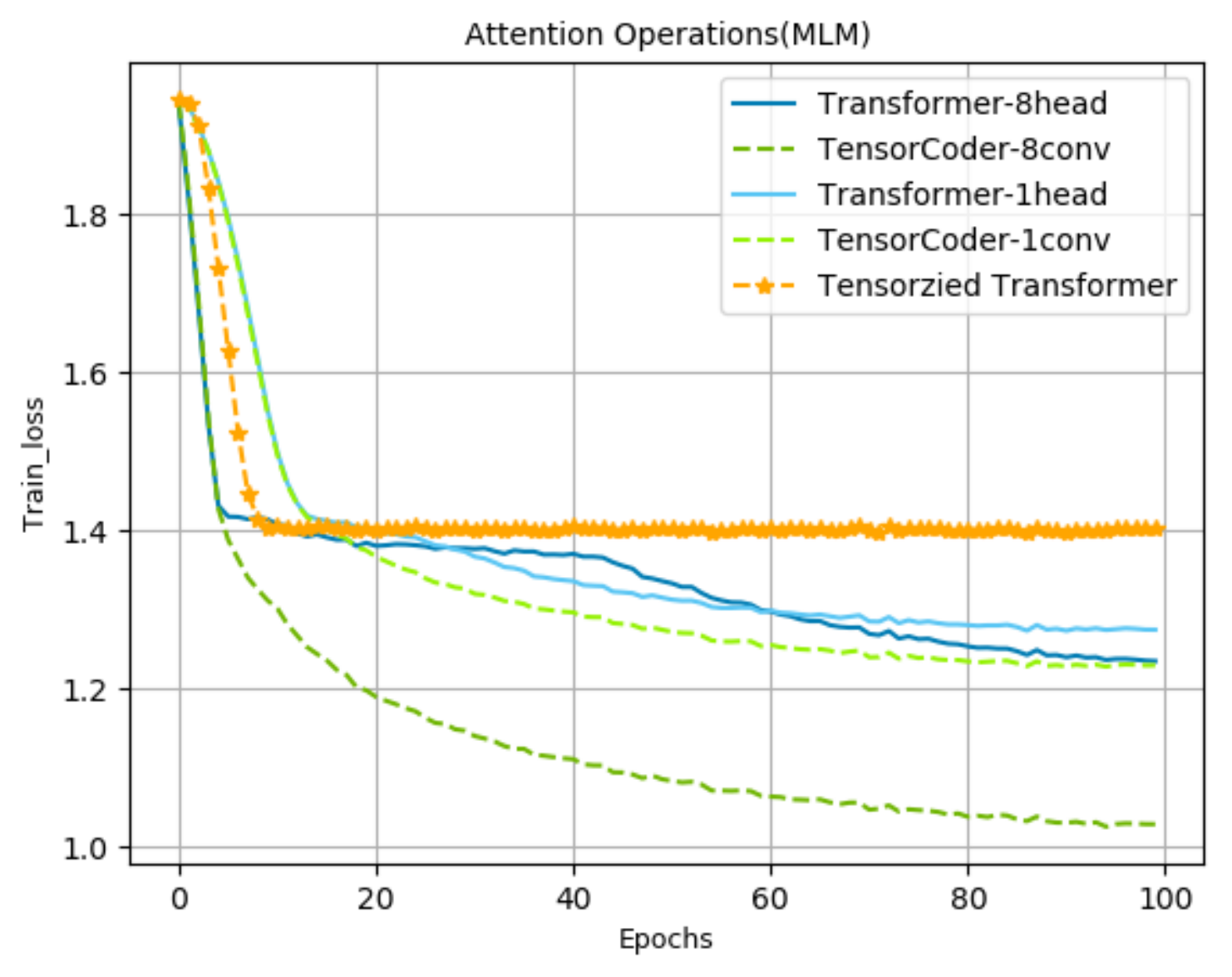}
    \end{minipage}%
    }
    \subfloat[Wiki-103 Loss]{
    \begin{minipage}[t]{0.5\linewidth}
    \centering
    \includegraphics[width=2.5in]{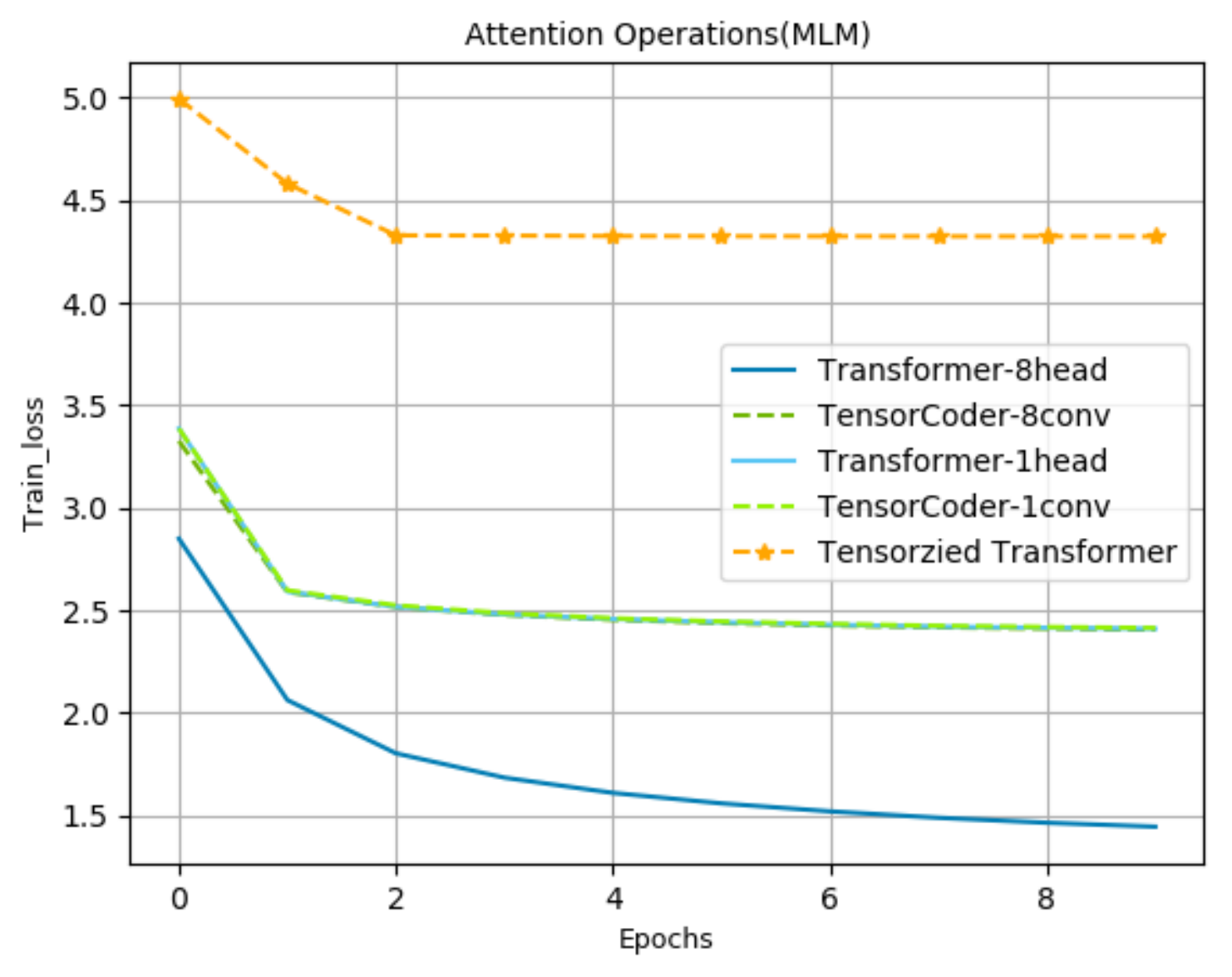}
    \end{minipage}%
    }
    \caption{ Loss of TensorCoder vs Transformer on PTB and Wiki-103 dataset.}
    \label{Loss}
\end{figure}
\begin{table}[tbp]\small
\renewcommand\arraystretch{1.2}
  \centering
  \begin{tabular}{cccccc}
    \toprule[1pt]
    \multirow{2}{*}{\textbf{Model}}& \multicolumn{2}{c}{\textbf{PTB}} & \multicolumn{2}{c}{\textbf{Wiki-103}}\\
    \cline{2-5}
    &valid Loss& Attention GFLOPs & valid Loss&Attention GFLOPs\\
    \hline
    Transformer-1head & 1.277 & 0.035 &2.420 & 0.047\\
    TensorCoder-1conv & 1.232 & 0.033 & 2.422 & 0.035\\
    Tensorized Transformer~\cite{ma2019a} &1.401 & 2.419 & 4.320 & 7.543\\
    \hline
   Transformer-8head~\cite{vaswani2017attention} & 1.242 & 0.284 &1.425 & 0.378 \\
    TensorCoder-8conv & 1.081& 0.093 & 2.401 &  0.103 \\
    \bottomrule[1pt]
  \end{tabular}
  \caption{Valid Loss and FLOPs on PTB and Wiki-103.}
  \label{MLM_FLOPs}
\end{table}
% \vspace{-0.cm}
%  Loss and FLOPs of TensorCoder vs Transformer on PTB. In right (b), the result '3.05x' is the ratio of FLOPs by 'Transformer-8head / TensorCoder-8conv'; the result '1.08x' is the ratio by 'Transformer-1head / TensorCoder-1conv'.The loss of (a) and the mark point of (b) is the same configuration, where $N=100, d_{head} = 40, L=3, d_{model}=320$.

The performance and efficiency of our model (TensorCoder) in the masked language modeling are shown in Figure~\ref{Loss} and Table \ref{MLM_FLOPs}. We use short sequences (the length is 100) on the smaller PTB dataset, and use long sequences (the length is 160) on the larger wiki-103 dataset.
% We can see that our model has better results in both two datasets, which shows that our model is all capable of modeling long and short sequences.
In PTB dataset, the two convolution methods of TensorCoder have lower loss compared with the Transformer, and the model complexity is lower, the maximum is $3.05$ times. In Wiki-103 dataset, the loss and FLOPs of TensorCoder are much smaller than Tensorized Transformer. 
% TensorCoder-1conv and Transformer-1head have similar loss and FLOPs. 
% In the two datasets, because of the limitation of sequence length, the FLOPs of Tensorcoder-1conv have no great advantage over Transformer-1head. 
More complexity comparison and hyperparameter settings can be found in Supplementary Materials \uppercase\expandafter{\romannumeral4} {model complexity comparison}.

% Loss and FLOPs of TensorCoder vs Transformer on WikiText-103. In right (b), the result `3.65x' is the ratio of FLOPs by `Transformer-8head / TensorCoder-8conv'; the result `1.32x' is the ratio by `Transformer-1head / TensorCoder-1conv'.The loss of (a) and the mark point of (b) is the same configuration, where $N=160, d_{head} = 40, L=4, d_{model}=150$.

% Table~\ref{result} and Table~\ref{Tabel1} show that our model get the lower PPL than other models in three datasets. An exciting observation is that our model has much fewer parameters.
% The model of Transformer-XL+TT~\cite{khrulkov2019tensorized} is a recent compression model with Tensor Train to compress the input embedding layers only. Sparse Transformer~\cite{Rewon2019sparse} use the method of sparse attention matrix to compress Transformer model. The results in Table~\ref{result} show that compared with Transformer-XL+TT, our method has much fewer parameters, and better language modeling performance. 
% These results verify that our model (i.e., Multi-linear attention) is effective in language modeling tasks, and has performed well for the model compression.
% Other details (such as hyperparameters and Hardware) can be found in Supplementary Materials E.

%that this method have more parameters than ours, but PPL increased in test and validation datasets than Transformer-XL~\cite{dai2019transformer}.
\subsection{Evaluation on Neural Machine Translation}
On the machine translation tasks, we report results on WMT16 English to German(En-De) and IWSLT16 English to German (En-De)~\cite{cettoloEtAl:EAMT2012} benchmark datasets. 
% For WMT14 En-DE, the training data consists of 4.5 million sentence pairs, We validated on newstest2013 and tested on newstest2014. We employed SentencePiece~\footnote{https://github.com/google/sentencepiece} to the sentences, with a 37K joint source and target vocabulary.
For IWSLT EN-De, we used some pre-processing steps to data cleaning. Specifically, for the dataset we used the training data that consists of 197k sentence pairs. We validated on tst2013 and tested on tst2014. Sentences were encoded using SentencePiece~\footnote{https://github.com/google/sentencepiece} , which has a shared source-target vocabulary of about 32k tokens. For WMT16 En-DE, We had a 37K joint source and target vocabulary.
For evaluation, we used beam search with a beam size of 5 and a length penalty $\alpha$=0.6. The BLEU scores use multi-bleu~\footnote{https://github.com/moses-smt/mosesdecoder/blob/master/scripts/generic/multi-bleu.perl} to compute.

In neural machine translation, we have replaced the attention with dimension-wise attention in each layer, with
other parts remaining the same. The FLOPs of Table~\ref{WMt16} and Table~\ref{IWSLT16} is the attention layers' FLOPs in encoder structure. The length of sequence is set to the maximum length of the sentence in the training set. '$*$' indicates that the result is our own implementation, '$-$' indicates no reported results in that setting.

\begin{table}[htbp]\small
\begin{minipage}[tbp]{0.5\linewidth}
  \renewcommand{\arraystretch}{1.2}
  \centering
  \setlength{\tabcolsep}{0.6mm}{
  \begin{tabular}{cccc}
  \toprule[1pt]    
    \textbf{Model} & \textbf{BLEU} & \textbf{FLOPs}\\
    \hline
    Linguistic Input Featurec~\cite{sennrich2016linguistic} & 28.4 & - \\
    STT Translation System~\cite{williams2016edinburgh} & 30.6 & -\\
    Tensorized Transformer~\cite{ma2019a} &34.1 & 0.242 \\
    \hline
   Transformer-1head~\cite{vaswani2017attention} & 30.4* & 0.02 \\
    TensorCoder-1conv & 30.7& 0.021 \\
    \hline
    Transformer-8head~\cite{vaswani2017attention} & 33.5* & 0.163 \\
    TensorCoder-8conv & 33.1& 0.058 \\
    \bottomrule[1pt]
  \end{tabular}}
  \caption{Bleu and FLOps in WMT16}
  \label{WMt16}
  \end{minipage}
\begin{minipage}[htbp]{0.5\linewidth}
  \renewcommand{\arraystretch}{1.2}
  \centering
  \setlength{\tabcolsep}{0.8mm}{
  \begin{tabular}{cccc}
    \toprule[1pt]
    \textbf{Model} & \textbf{BLEU} & \textbf{FLOPs}\\
    \hline
    NAT-FT~\cite{gu2018non} & 26.5 & - \\
    imitate NAT~\cite{wei2019imitation} & 31.8 & -\\
    Iterative Refinement NAT~\cite{lee2018deterministic} &30.1 & - \\
    \hline
    Transformer-1head~\cite{vaswani2017attention} & 30.3* & 0.054 \\
    TensorCoder-1conv & 30.1 & 0.051 \\
    \hline
    Transformer-8head~\cite{vaswani2017attention} & 33.1* & 0.432 \\
    TensorCoder-8conv & 32.5 & 0.144 \\
    \bottomrule[1pt]
  \end{tabular}}
  \caption{Bleu and FLOps in IWSLT16}
  \label{IWSLT16}
  \end{minipage}
  \vspace{-0.1cm}
\end{table}

In Table~\ref{WMt16} and Table~\ref{IWSLT16}, we select Transformer and Tensorized Transformer as the baseline. In WMT16 dataset, the bleu score of TensorCoder (1conv and 8conv) are 30.7 and 33.1, which are similar performance over Transformer. However, TensorCoder has lower FLOPs, the FLOPs of Transformer is 3.5 times over our model under similar performance. In IWSLT16, The TensorCoder aslo has lower FLOPs, the FLOPs of TensorCoder only is 1/3 of Transformer.

% 4 times

% \subsection{Discussion}
% We have showed the results in language modeling and neural machine translation tasks using the Multi-linear attention. Our experimental design is based on the better experimental results(i.e.,perplexity), then to compares the model parameters. Our method can greatly compression the model parameters in some datasets and achieve comparable or better results with the state of the art results. Regarding the rationale for the improvements, besides the alleviation of overfitting by reducing parameters, another reason is that our method captures more information than the original Transformer. In Corollary~\ref{corollary}, we prove that the output of the original attention can be represented by summing over the $3$-order tensor. In Figure~\ref{Model}, we use a concat function over these matrices from tensor splitting. The operation of concat models all values in the $3$-order tensor, and thus captures more information than the operator of sum. In order to further illustrate the effectiveness of our model, we also add a lot of groups comparison experiments. These results and further analyses can be found in Supplementary Materials E.4.

\section{Conclusion}
In the paper, we propose an encoder-decoder language model, namely TensorCoder, which replaces the token-wise attention of Transformer with the dimension-wise attention. Compared with token-wise attention, the complexity of dimension-wise attention is $O(Nd^2)$. When $N$ increases, the complexity of dimension-wise attention increases linearly rather than quadratically.
% TensorCoder have 
% TensorCoder also model dimension-wise and token-wise attention information.
The experimental results show that TensorCoder is competent in long-sequence tasks and helps to improve the efficiency of pre-trained language models in resource-limited environments.

% \section{Acknowledgement}
% This work is supported in part by the state key development program of China (grant No. 2017YFE0111900), Natural Science Foundation of China (grant No. U1636203, 61772363, 2018YFC0831704), and the European Unions Horizon 2020 research and innovation programme under the Marie SkodowskaCurie grant agreement No.721321. 

\section{Impact Statement}

\emph{Transformer} which is a fundamental \emph{Sequence-to-Sequence (Seq2Seq)} deep learning model is now widely applied to many \emph{natural language processing (NLP)} tasks due to its capacity of modeling long-range dependencies and its parallelizability.   Despite its marvelous successes on many NLP tasks,
% (esp., \emph{Neural Machine Translation (NMT)} and \emph{Pre-trained Language Models}),
the limitations of Transformer are also obvious. Firstly, the attention mechanism in Transformer involves only the token-wise attention and is not aware of the interactions between different dimensions (i.e., features). Secondly, the computation time of the token-wise attention in Transformer is quadratic to the sequence length which renders it not scalable to long-sequence tasks and not applicable to devices with limited resources. 

Motivated by this, we proposed \emph{TensorCoder} in this paper. Instead of directly computing the token-wise attention, our TensorCoder deploys dimension-wise attention which allows the token representations learned fuses the interactions between different dimensions and also captures the same output representation of the token-wise attention. And remarkably, the time complexity of TensorCoder is linear to the sequence length and the benefits are two-folds. Firstly, it could be efficiently applied to many long-sequence tasks and even potentially the document-level NLP tasks (e.g., the document-level Neural Machine Translation and the document-level Neural Language Generation such as fiction generation). Secondly, due to the reduction of the computation and memory consumption compared with Transformer, our TensorCoder could be deployed in the scenarios where the computation resource is limited (e.g., edge devices such as mobile phones, tablets, etc.). 

In the future, we will further study the universality of TensorCoder. Specifically, we will study the relationship between the dimensionality of the dimension-wise attention and the sequence length and provide an insight and guidance on how to set the dimensionality for a given sequence length. 

%Transformers and their derived pre-trained language model have been widely recognized and applied. They lead the development trend of Natural language processing, and gradually apply to other fields. But large-scale pre-trained language models also bring many limitations, such as hardware resources and time cost. Our work focuses on reducing the complexity of the Transformer structure to obtain a new framework with low complexity and high performance. 

%TensorCoder is a unified encoder-decoder framework. It has a lighter attention mechanism, which is linear to the length of sequence. So it can handle long sequence tasks with a low resource consumption, even be used in music, images, and other fields that have longer sequence. And it can simultaneously model the explicit and implicit attention information of the text sequence. Therefore it can be used to improve some NLP tasks. TensorCoer have a encoder-decoder framework, which is a general framework and can be used easily in various forms. Also, TensorCoder can be used as the core module of a pre-trained language model to develop a language model that can handle long text and multiple tasks.

%We will carry out further research to verify the universality of TensorCoder. In addition, we recommend that researchers pay attention to the relationship between the dimensions of dimension-wise attention and the length of the sequence. As the length of the sequence processed increases, how should the dimensions be selected, so that the performance of TensorCoder is not average?

\bibliographystyle{plain}

%\bibliography{neurips_2020}
%\bibliographystyle{plain}
\end{document}